\documentclass[preprint,12pt]{elsarticle}

\usepackage{amssymb}
\usepackage{mathrsfs}
\usepackage{mathrsfs}
\usepackage{mathrsfs}
\usepackage{amsfonts}
\usepackage{amsfonts}
\usepackage{amsfonts}
\usepackage{mathrsfs}
\usepackage{amsmath}
\usepackage{graphicx}
\usepackage{multirow}
\usepackage{caption2}
\usepackage{float}
\usepackage{booktabs}
\usepackage{algorithm}
\usepackage{algorithmic}
\usepackage{subfigure}
\usepackage{color}
\newtheorem{theorem}{Theorem}[section]

\numberwithin{equation}{section}

\newtheorem{corollary}[theorem]{Corollary}

\newtheorem{lemma}[theorem]{Lemma}

\newenvironment{proof}[1][Proof]{\textbf{#1. }}{\ \rule{0.5em}{0.5em}}%

\journal{}

\begin{document}

\begin{frontmatter}



\title{Optimal Approximation and Learning Rates for Deep Convolutional Neural Networks\tnoteref{t1}} \tnotetext[t1]{The research was partially
	supported by the  National Key R\&D Program of China (No.2020YFA0713900) and the
    Natural Science Foundation of China [Grant No  62276209]. Email: sblin1983@gmail.com }


 \author{Shao-Bo Lin} 
\address{Center for Intelligent Decision-Making and Machine Learning, School of Management, Xi'an Jiaotong University, Xi'an 710049, China
   }

\begin{abstract}
This paper focuses on approximation and learning performance analysis for deep convolutional neural networks   with zero-padding and max-pooling. We prove that,   to approximate $r$-smooth function, the approximation rates of deep convolutional neural networks with depth $L$ are of order $ (L^2/\log L)^{-2r/d} $, which is optimal up to a logarithmic factor. Furthermore, we deduce almost optimal learning rates for implementing empirical risk minimization over deep convolutional neural networks.
\end{abstract}

\end{frontmatter}

\section{Introduction}
Deep learning \cite{lecun2015deep} has made great breakthrough  and profound impacts   in numerous application regions including the  computer science, life science  and  management science. One of the most important reasons for its success   is the architecture  (or structure) \cite{goodfellow2016deep}   which autonomously encodes the a-priori information in the network and  significantly  reduces the number of free parameters simultaneously to improve the learning performance.  Deep convolutional neural networks (DCNNs), a widely used structured  deep neural networks,  have been triggered enormous research activities in both applications \cite{gonzalez2018deep,rawat2017deep,yoo2015deep} and theoretical analysis \cite{zhou2020theory,fang2020theory,mao2021theory}.

In this paper, we focus on approximation and learning performance analysis for DCNNs induced by the rectifier linear unit (ReLU) $\sigma(t)=\max\{t,0\}$.  
  For  $\vec{v}\in\mathbb R^{d'}$,   the one-dimensional and one-channel  convolution is defined by
\begin{equation}\label{convolution1}
    (\vec{w}\star\vec{v})_j=\sum_{k=j-s}^{j}w_{j-k}v_{k+s},\qquad j=1,\dots,d'-s, 
\end{equation}
where  $\vec{w}=(w_j)_{j=-\infty}^\infty$ is a filter of   length $s$, i.e.  $w_j\neq0$ only for $0\leq j\leq s$.
Then the classical DCNN is given by \begin{equation}\label{cDCNN}
     \mathcal N_{L,s}^{\star}(x)
     := \vec{a}_L\cdot  \sigma\circ \mathcal C_{L,\vec{w}^L,\vec{b}^L}^{\star} \circ \sigma  \circ \dots \circ \sigma\circ\mathcal C^{\star}_{1,\vec{w}^1,\vec{b}^1}(x),\qquad x\in\mathbb R^d,
\end{equation}
where  the filter $\vec{w}^\ell$ is of length $s$,  bias vector  $\vec{b}^\ell\in\mathbb R^{d_\ell}$ with $d_\ell=d_{\ell-1}-s$ and $d_0=d$,
$$  
    \mathcal C_{\ell,\vec{w}^\ell,\vec{b}^\ell}^\star(\vec{u}):=\vec{w}^\ell\star\vec{u}+\vec{b}^\ell,\qquad  \vec{u}\in\mathbb R^{d_{\ell-1}},
$$
  and  $\sigma$ acts on vectors componentwise. Since DCNN defined by \eqref{cDCNN}  possesses a contracting nature in the sense that the width of the network shrinks with respect to the depth, i.e., all the width of DCNN is smaller than $d$, showing that DCNN is not a universal approximant since 
the minimal width requirement for the universality of deep neural networks  is $d+1$ \cite{hanin2019universal}. 

Zero-padding is a feasible approach to avoid the aforementioned non-universality of DCNN.     Define  the convolution with zero-padding by
\begin{equation}\label{convolution}
    (\vec{w}*\vec{v})_j=\sum_{k=1}^{d'}w_{j-k}v_k,\qquad j=1,\dots,d'+s 
\end{equation}
and corresponding  DCNN with zero-padding  by
\begin{equation}\label{eDCNN}
     \mathcal N_{L,s}(x)
     = \vec{a}_L\cdot \sigma\circ \mathcal C_{L,\vec{w}^L,\vec{b}^L} \circ \sigma  \circ \dots \circ \sigma\circ\mathcal C_{1,\vec{w}^1,\vec{b}^1}(x),\qquad x\in\mathbb R^d,
\end{equation}
where 
\begin{equation}\label{convolution-operator}
   \mathcal C_{\ell,\vec{w}^\ell,\vec{b}^\ell}(\vec{u}):=\vec{w}^\ell* \vec{u}+\vec{b}^\ell, \qquad \vec{u}\in\mathbb R^{d_{\ell-1}}
\end{equation}
with  $d_\ell=d_{\ell-1}+s$. As DCNN with zero-padding defined by \eqref{eDCNN} exhibits an expansive nature, we write it as eDCNN  for the sake of brevity and denote by $\mathcal H_{L,s}$ the set of all  eDCNNs.  
 The power of zero-padding has been explored in  \cite{zhou2020universality}, where the universal approximation of eDCNN without additional fully connected layers has been established. Furthermore, it was shown  in  \cite{han2023deep} that the translation-equivalence of DCNN  can  also be enhanced by zero-padding. The problem is, however, that    optimal approximation   and learning rates for eDCNN remain open, although some sub-optimal results have been presented in \cite{zhou2018deep,zhou2020theory,zhou2022learning,lin2022universal}.

The purpose of this note is  to derive optimal approximation and learning rates for eDCNN. We prove that, to approximate the well known $r$-smooth functions, equipped with the well known max-pooling scheme, eDCNN succeed in yielding an approximation rate of order $ (L^2/\log L)^{-r/d}$, which is essentially better than the existing rates     $ L^{-r/d}$   for both eDCNN and eDCNN with pooling established in \cite{zhou2006approximation,zhou2020theory}. We also prove that the derived approximation rate for eDCNN cannot be essentially improved up to a logarithmic factor. Based on the derived (almost) optimal approximation rate, we  deduce (almost) optimal learning rates  for implementing empirical risk minimization (ERM) over eDCNN, which shows that   eDCNN is also one of the most powerful tools for the learning purpose. This together with the   translation-equivalence of eDCNN discussed  in \cite{han2023deep} shows the power of the convolutional structure in deep learning.

\section{Approximation Rates Analysis}

As the width of eDCNN increases linearly with respect to the depth, it is preferable  to adopt some pooling schemes \cite{zhou2020theory} to shrink the network size, among which max-pooling is the most popular.
 Mathemtaically, the max-pooling operator $\mathcal S_{d',u}:\mathbb R^{d'}\rightarrow \mathbb R^{[d'/u]}$ for a vector $\vec{v}\in\mathbb R^{d'}$ with   pooling size   $u$   is defined by 
\begin{equation}\label{def.down-sampling}
    \mathcal S_{d',u}(\vec{v})=(\max_{j=1,\dots,u}v_{(k-1)u+j})_{k=1}^{[d'/u]},  
\end{equation}
where $[a]$ denotes the integer part of the real number $a$. We force the size of eDCNN to be independent of the number of layers  by using the max-pooping operator defined by \eqref{def.down-sampling}. 

Write  $L_{1,\max}:=\left\lceil\frac{(2d+10)d}{s-1}\right\rceil$, $d_{1,\max}:=d+L_{1,\max}s$ and
$
    d_{\max}:=2d+10+\left\lceil\frac{(2d+10)^2}{s-1}\right\rceil s.
$
Define the max-pooling scheme as
\begin{equation}\label{max-pooling}
    \mathcal P_{\max}(\vec{v}):=\left\{
   \begin{array}{cc}
   \mathcal S_{d_{1,\max},d}(\vec{v}) & \mbox{if}\ |\vec{v}|= d_{1,\max}\ \mbox{and}\ \ell\leq L_{1,\max},\\
     \mathcal S_{d_{\max},2d+10}(\vec{v})  
     & \mbox{if}\ |\vec{v}|= d_{\max},    \\
      \vec{v}  & \mbox{otherwise},
   \end{array}\right.
\end{equation}
which means that  the output of the $L_{1,\max}$-th layer is max-pooled with pooling size $d$ and the output of the layers with sizes   $d_{\max}$ are pooled with   pooling size $2d+10$. 
We then define eDCNN with max-pooling by
\begin{equation}\label{eDCNN-max-pooling}
    \mathcal N_{L,s}^{pool}(x):=
    \vec{a}\cdot \mathcal P_{\max}\circ \sigma\circ\mathcal C_{L,\vec{w}^L,\vec{b}^L}\circ \dots\circ \mathcal P_{\max}\circ
    \sigma\circ
    \mathcal C_{1,\vec{w}^1,\vec{b}^1}(x).
\end{equation}
Denote by $\Phi_{L,s}^{pool}$ the set of eDCNNs formed as \eqref{eDCNN-max-pooling}. 
It is easy to find that the width of the network defined in \eqref{eDCNN-max-pooling} is always smaller than  $d_{\max}$. It should be highlighted that the network structure in \eqref{eDCNN-max-pooling}  is fixed once  $s$ and $d$ are specified and there are totally $\mathcal O(L)$ free parameters in $\mathcal N_{L,s}^{pool}$. Our following result shows that even for $\mathcal O(L)$ tunable parameters, eDCNN with max-pooling can approximate $r$-smooth function with an order of $(L^2/\log L)^{-r/d})$, which is essentially better than $\mathcal O(L^{-r/d})$, the best rates for shallow approximation \cite{yarotsky2017error,guo2019realizing} and existing approximation rates for eDCNN \cite{zhou2020universality,zhou2020theory}.  

Let $\mathbb I^d:=[0,1]^d$. For  $c_0>0$ and
$r=s+\mu$ with $s\in\mathbb N_0:=\{0\}\cup\mathbb N$ and $0<\mu\leq 1$,  $f:\mathbb I^d\rightarrow\mathbb R$ is said to be 
$(r,c_0)$-smooth if $f$ is $s$-times differentiable and its $s$-th partial derivative satisfies
the Lipschitz condition
\begin{equation}\label{lip}
          \left|\frac{\partial^sf}{\partial x_1^{\alpha_1}\dots\partial
          x_d^{\alpha_d}}
          (x)-\frac{\partial^sf}{\partial x_1^{\alpha_1}\dots\partial
          x_d^{\alpha_d}}
          (x')\right|\leq c_0\|x-x'\|_2^\mu,\quad \forall\ x,x'\in\mathcal I^d 
\end{equation}
for every
$\alpha_j\in \mathbb N_0$, $j=1,\dots,d$ with
$\alpha_1+\dots+\alpha_d=s$. 
Denote by $Lip^{(r,c_0)}$ the set of all  
$(r,c_0)$-smooth functions and $Lip^{(r,c_0)}_M:=\{f\in Lip^{(r,c_0)}:\|f\|_{L^\infty(\mathbb I^d)}\leq M\}$. We then present our first main result as follows. 

\begin{theorem}\label{Theorem:optimal-app}
Let $r,c_0>0$, $s\geq 2$ and $M>0$. There holds
\begin{equation}\label{almost-optimal-app}
 C_1(L\log L)^{-\frac{2r}d} \leq   \mbox{dist} \left(Lip^{(r,c_0)}_M,\pi_M\Phi_{L,s}^{pool},L^\infty(\mathbb I^d)\right)\leq C_2 (L/\log L)^{-\frac{2r}d},
\end{equation}
where $\pi_M\mathbb A=\{\pi_Mf:f\in\mathbb A\}$ with $\pi_Mf(x)=\mbox{sign}(f(x))\max\{|f(x)|,M\}$ denotes the truncation of the set $\mathbb A$, $C_1,C_2$ are constants depending only on $r,c_0,s,d$ and 
$$
   \mbox{dist}(\mathbb A, \mathbb B,L^\infty(\mathbb I^d)):=\sup_{f\in \mathbb A}\mbox{dist}(f,\mathbb B,L^\infty(\mathbb I^d)):=\sup_{f\in \mathbb A}\inf_{g\in\mathbb B}\|f-g\|_{L^\infty(\mathbb I^d)}
$$
for $\mathbb A,\mathbb B\subseteq L^\infty(\mathbb I^d)$.
\end{theorem}

It can be found in Theorem \ref{Theorem:optimal-app} that up to a logarithmic factor, the derived approximation rates   are optimal. Recalling that there are totally $\mathcal O(L)$ parameters involved in eDCNN with  max-pooling, the approximation rates of order $ (L^2/\log L)^{-r/d}$ shows the power of  depth since the optimal approximation rates for shallow nets with  $N$ parameters are only of order $ N^{-r/d} $.   Furthermore, the derived approximation rates are much better than those in \cite{zhou2020universality,zhou2020theory,lin2022universal,han2023deep}, in which the approximation  rates for eDCNNs are of order  $L^{-r/d}$. 

Similar results for  deep fully connected networks (DFCNs)  have been presented in  \cite{yarotsky2020phase,lu2021deep}. In particular, approximation rates of order $ (L^2/\log L)^{-r/d}$ have been established  for DFCNs with fixed width $2d+10$  and rates of order $((NL)^2/\log (NL))^{-r/d}$ for DFCNs with width $N$ and depth $L$, where $N\geq C_d$ for $C_d$ a constant depending only on $d$. Theorem \ref{Theorem:optimal-app} shows that eDCNNs with max-pooling perform at least not worse than DFCNs. However, it has been verified in \cite{han2023deep} that  eDCNN succeeds in encoding the translation-equivalence into the network structure   which is  beyond the capability of DFCNs. Theorem  
 \ref{Theorem:optimal-app}  together with the results in \cite{han2023deep} thus rigorously shows the power of the convolution structure over  fully connection.

\section{ Learning Rates Analysis}
In this section, we study the  learning performance eDCNN with max-pooling. Our analysis is carried out in the standard least square regression framework \cite{gyorfi2002distribution}, in which
the { samples} $D:=\{(x_i,y_i)\}_{i=1}^{m}$ are assumed to be drawn independently and identically according to an unknown but definite  distribution $\rho:=\rho_X\times\rho(y|x)$ with $\rho_X$ the marginal distribution and $\rho(y|x)$ the conditional distribution conditioned on $x$. Our aim  is to find an estimator in $\Phi_{s,L}^{pool}$ to approximate the  well known regression function $f_\rho=\int_{\mathcal Y}yd\rho(y|x)$, which minimizes the generalization error $\mathcal E(f):=\int(f(x)-y)^2d\rho$. 
Denoting by $L_{\rho_X}^2$ the space of $\rho$-square integrable functions endowed with norm $\|\cdot\|_\rho$, it is easy to check 
\begin{equation}\label{equality}
      \mathcal E(f)-\mathcal (f_\rho)=\|f-f_\rho\|_\rho^2,\qquad f\in  L_{\rho_X}^2.
\end{equation}
Define 
\begin{equation}\label{ERM-eDCN}
   f_{D,L,s}^{pool}\in\arg\min_{f\in \Phi_{L,s}^{pool}}\frac1m\sum_{i=1}^m(f(x_i)-y_i)^2 
\end{equation}
to be an arbitrary  global minimum of the least squares problem. We  
 present our second main result in the following theorem.

 \begin{theorem}\label{Theorem:optimal-learning}
Let $r,c_0>0$, $s\geq 2$, $M>0$ and $0<\delta<1$. If $|y_i|\leq M$, $f_\rho\in Lip^{(r,c_0)}$ and 
$
     L\sim L^{\frac{d}{4r+2d}},
$
then with confidence $1-\delta$ there holds
\begin{equation}\label{learning-rate}
   \mathcal E(\pi_Mf_{D,L,s}^{pool})-\mathcal E(f_\rho) \leq C_3 m^{-\frac{2r}{2r+d}}(\log m)^{\max\left\{\frac{2r}d,2\right\}},
\end{equation}
where  $C_3$ is a constant  depending only on $r,c_0,s,d$ and $M$.
\end{theorem}

If $f_\rho\in Lip^{(r,c_0)}$, it can be found in \cite[Chap.3]{gyorfi2002distribution} that there is not any learning algorithm based on $D$ such that the learning rate is essentially better than $\mathcal O(m^{-\frac{2r}{2r+d}})$. Theorem \ref{Theorem:optimal-learning} thus shows that implementing ERM over eDCNN is one of the best learning algorithms in learning smooth regression functions. In particular, we can derive the following corollary.

\begin{corollary}\label{Corollary:optimal-learning}
Let $r,c_0>0$, $s\geq 2$ and $M>0$. If $|y_i|\leq M$ and 
$
     L\sim L^{\frac{d}{4r+2d}},
$
then
\begin{equation}\label{almost-optimal-learning}
 C_4 m^{-\frac{2r}{2r+d}} \leq  \sup_{f_\rho\in Lip^{(r,c_0)}}E\left[\|\pi_Mf_{D,L,s}^{pool}-f_\rho\|_\rho^2\right]\leq C_5 m^{-\frac{2r}{2r+d}}(\log m)^{\max\left\{\frac{2r}d,2\right\}},
\end{equation}
where  $C_4,C_5$ are constants depending only on $r,c_0,s,d$ and $M$.
\end{corollary}

Though similar optimal learning rates have been derived for DFCNs \cite{schmidt2020nonparametric,han2020depth,chui2020realization}, it remains open whether learning with eDCNN can achieve the same  rates. Indeed, only universal consistency   for learning with eDCNN has been verified in \cite{lin2022universal} and sub-optimal learning rates of order $m^{-r}{2r+d}$ have been deduced in \cite{mao2021theory,zhou2022learning}. Theorem \ref{Theorem:optimal-app} and Corollary \ref{Corollary:optimal-learning} present  a  new record for learning with eDCNN.  
It should be mentioned that  optimal learning rates for eDCNNs with hybrid structure involving both convolutional layer and  fully connected layers have been established in \cite{fang2023optimal}. The main novelty  of our results is that there are not any fully connected layers being added to the network, making the analysis take the depth  $L$ as the only parameter.

\section{Proofs}
To prove Theorem \ref{Theorem:optimal-app}, we need several preliminary lemmas. The first one provided in  \cite[Theorem 4.1]{yarotsky2020phase} shows the approximation rates of DFCN with fixed width $2d+10$.
 
\begin{lemma}\label{Lemma:fixed-width}
    Let $r,c_0>0$, there exists a constant $\tilde{C}_1$ depending only on $r$ and $d$ such that 
$$
    \mbox{dist}\left(Lip^{(r,c_0)},\Psi_{L,2d+10,\dots,2d+10},L^\infty(\mathbb I^d)\right)\leq \tilde{C}_1L^{-2r/d}\log^{2r/d} L,
$$
where $\Psi_{L,2d+10,\dots,2d+10}$ is the set of DFCNs with depth $L$ and width $2d+10$ in each layer. 
\end{lemma}

Let $\nu$ be a probability measure on $\mathbb I^d$. For a  function $f:\mathbb I^d\rightarrow\mathbb R$, set
$\|f\|_{L^p(\nu)}:=\left\{\int_{\mathbb I^d}|f(x)|^pd\nu\right\}^p$. Denote by $L^p (\nu)$ the set of all functions satisfying $\|f\|_{L^p(\nu)}<\infty$.
For $\mathcal V\subset L^p (\nu)$, denote by
$\mathcal N(\epsilon,\mathcal V,\|\cdot\|_{L^p(\nu)})$ the covering number \cite[Def. 9.3]{gyorfi2002distribution} of $\mathcal V$ in $L^p(\nu)$, which is
 the number of elements in a least $\varepsilon$-net of $\mathcal V$ with respect to $\|\cdot\|_{L^p(\nu)}$. In particular, denote by $\mathcal N_p(\epsilon,\mathcal V,x_1^m):=\mathcal N(\epsilon,\mathcal V,\|\cdot\|_{L^p(\nu_m)})$ with $\nu_m$ the empirical measure with respect to $x_1^m=(x_1,\dots,x_m)\in (\mathbb I^d)^m$.  
The second lemma that can be derived by the same method as \cite[Lemma 4]{lin2022universal} builds the covering number estimate of $\Phi_{L,s}^{pool}$.

\begin{lemma}\label{Lemma:covering-number}
Let $\nu$ be a probability measure on $\mathbb I^d$. For any $0<\varepsilon\leq M$, there holds
$$
   \log_2\mathcal N_1(\epsilon,\pi_M\Phi_{L,s}^{pool}, L^1 (\nu))\leq c^* L^2\log L \log\frac{M}\epsilon,
$$
where $c^*$ is a  constant depending only on $s$ and $d$.
\end{lemma}

The next lemma derived in \cite{guo2019realizing} presents a relation between the covering number and approximation.

\begin{lemma}\label{Lemma:Relation c and l}
Let  $n\in\mathbb N$ and $V\subseteq L_1(\mathbb I^d)$. For arbitrary
$\varepsilon>0$, if
\begin{equation}\label{covering condition}
  \mathcal N(\varepsilon,V,L_1(\mathbb I^d))\leq
  \tilde{C}_1\left(\frac{\tilde{C_2}n^\beta}{\varepsilon}\right)^n
\end{equation}
with $\beta,\tilde{C}_1,\tilde{C}_2\geq0$, then
\begin{equation}\label{lower bound deep}
   \mbox{dist}(Lip^{(r,c_0)}_M,V,L_1(\mathbb I^d))\geq C'
       (n\log_2(n+1))^{-r/d},
\end{equation}
where  $C'$ is a constant independent of $n$ or $\varepsilon$.
\end{lemma}
 
The fourth lemma is the 
 convolution   factorization lemma  provided in  \cite[Theorem 3]{zhou2020universality}.

\begin{lemma}\label{Lemma:convolution-fraction}
Let $S\geq 0,2\leq s\leq d$ and $\vec{u}=(u_k)_{-\infty}^\infty$ be supported on $\{0,\dots,S\}$. Then there exists $L<\frac{S}{s-1}+1$ filter vectors $\{\vec{w}^{\ell}\}_{\ell=1}^L$ supported on $\{0,\dots,s\}$ such that $\vec{u}=\vec{w}^{L}*\cdots*\vec{w}^{1}$.
\end{lemma}

 Define
\begin{equation}\label{convolution-matrix}
         T_{\tilde{d},d'}^{\vec{w}}:=\left[\begin{array}{ccccccc}
                w_0 & 0 & 0 & \cdots & 0\\
                w_1 & w_0& 0 & \cdots & 0\\
                \vdots &  \ddots &  \ddots & \ddots &\vdots\\
                w_{d'-1} & w_{d'-2} & \cdots&\cdots&w_0\\
                w_{d'}& w_{d'-1}& \cdots&  0\cdots & w_1\\
                 \vdots & \ddots & \ddots  & \ddots &\vdots\\
                 w_{\tilde{d}-d'}& \cdots&\cdots& \cdots& w_{\tilde{d}-2d'+1}\\
                  \vdots & \ddots &   \ddots & \ddots &\vdots\\
                    w_{\tilde{d}-2}&w_{\tilde{d}-3} &\cdots& w_{\tilde{d}-d'}&w_{\tilde{d}-d'-1}\\
                  w_{\tilde{d}-1}&w_{\tilde{d}-2}&\cdots&w_{\tilde{d}-d'+1}&w_{\tilde{d}-d'}
                  \end{array}\right].
\end{equation}
 Observe that if $W$ is supported in $\{0,\dots,S\}$ for some $S\in\mathbb N$, the entry $(T_{D,d'}^{\Vec{w}})_{i,k}=w_{i-k}$ vanishes when $i-k>S$. The fifth lemma 
proved in \cite{zhou2018deep} establishes the relation between the convolution and matrix multiplication. 

\begin{lemma}\label{Lemma:convolution-to-matrix}
 Let $2\leq s\leq d'$, $d'_\ell=d'+\ell s$ and $d'_0=d'$. If $\{\vec{w}^{\ell}\}_{\ell=1}^{L}$ is supported on  $\{0,\dots,s\}$, then
\begin{equation}\label{matrix-fraction}
    T_{d_\ell',d'}^{\vec{w}^\ell*\cdots \vec{w}^2*\vec{w}^1}=T_{d'_\ell,d'_{\ell-1}}^{\vec{w}^\ell}\cdots T_{d_2',d'_1}^{\vec{w}^2}T_{d'_1,d'}^{\vec{w}^1}
\end{equation}
holds for any $\ell\in\{1,2,\dots,L\}$.
\end{lemma}

 For a sequence $\vec{w}$ supported on $\{0,1,\dots,s\}$, write $\|\vec{w}\|_1=\sum_{k=-\infty}^\infty|w_k|$ and $\|\vec{w}\|_\infty=\max_{-\infty\leq k\leq \infty}|w_k|$.   Define $B^0:=\max_{x\in\mathbb I^d}\max_{k=1,\dots,d}|x^{(k)}|$ and 
$$
    B^\ell:=\|\vec{w}^\ell\|_1B^{\ell-1}\cdots B^1B^0,\qquad   \ell\geq 1.
$$
 Then for any $j=1,\dots,d_\ell$, direct computation yields
\begin{equation}\label{bound-1}
     \max_{x\in\mathbb I^d}\left|\left(T_{d_{\ell},d_{\ell-1}}^{\vec{w}^{\ell}}\dots
     T_{d_{1},d_{0}}^{\vec{w}^{1}} x\right)_j\right|\leq B^\ell 
\end{equation}
and for $1\leq k\leq \ell-1$
\begin{equation}\label{bound-1.1}
      \left|\left(T_{d_{\ell},d_{\ell-1}}^{\vec{w}^{\ell}}\dots
     T_{d_{k+1},d_{k}}^{\vec{w}^{k+1}}  B^k\mathbf{1}_{d_k}\right)_j\right|\leq B^\ell.
\end{equation}
For any $d'_0=d'\in\mathbb N$ and $d_\ell'=d'+\ell s$, define  the   restricted convolution operator by 
\begin{equation}\label{convolutional-mapping-res}
    \mathcal C^R_{\ell,\vec{w}^\ell,b^\ell}(x):=\vec{w}^\ell* x+b^\ell {\bf 1}_{d'_\ell}
\end{equation}
for $\vec{w}^\ell$ supported on $\{0,1,\dots,s\}$,    $b^\ell\in\mathbb R$ and $\mathbf 1_{d'_\ell}=(1,\dots,1)^T\in\mathbb R^{d'_\ell}$.
 The following lemma derived in \cite{han2023deep} presents 
quantifies the relation between deep convolution neural network and matrix multiplication.

\begin{lemma}\label{Lemma:Induction}
Let $\ell\in\mathbb N$, $2\leq s\leq d$ and $\mathcal C^R_{\ell,\vec{w}^\ell,b^\ell}$ be defined by \eqref{convolutional-mapping-res} with $\vec{w}^\ell$ supported on $\{0,1,\dots,s\}$ and $b^\ell=2^{\ell-1}B^\ell$, then  
\begin{eqnarray}\label{induction}
 \sigma\circ \mathcal C^R_{\ell,\vec{w}^\ell, {b}^\ell} \circ \sigma  \circ \dots \circ \sigma\circ\mathcal C^R_{1,\vec{w}^1, {b}^1}(x)
 & =&
 T_{d_\ell,d_{\ell-1}}^{\vec{w}^\ell}\cdots T_{d_2,d_1}^{\vec{w}^2}T_{d_1,d}^{\vec{w}^1}x +b^\ell{\bf 1}_{d_\ell}\nonumber\\
 &+&
 \sum_{k=1}^{\ell-1}T_{d_\ell,d_{\ell-1}}^{\vec{w}^\ell}\cdots T_{d_{k+1},d_{k}}^{\vec{w}^{k+1}}
 b^k{\mathbf 1}_{d_k}.
\end{eqnarray}
\end{lemma}

By the help of the above lemmas, we are in a position to prove Theorem \ref{Theorem:optimal-app} as follows. The  proof of the upper bound   is motivated by \cite{zhou2020universality,han2023deep} while the lower bound is similar as that in \cite{guo2019realizing}. 

\begin{proof}[Proof of Theorem \ref{Theorem:optimal-app}]
Given a $d'\times\tilde{d}$ matrix $W$, write
\begin{eqnarray*}
    \vec{u}^T 
     = (W_{1,1},W_{2,1},\dots,W_{d\tilde{d},1},W_{1,2},\dots,W_{\tilde{d},d})
     =: 
    (W_{0},\dots,W_{d'\tilde{d}-1}).
\end{eqnarray*}
by stacking the rows of $W$.
It follows from  Lemma \ref{Lemma:convolution-fraction} with $S=d'\tilde{d}-1$ and $s\geq 2$ that there exist  $\hat{L}<\frac{d'\tilde{d}-1}{s-1}+1$ filter vectors $\{\vec{w}^\ell\}_{\ell=1}^{\hat{L}}$ satisfying 
$
    \vec{u}= \vec{w}^{\hat{L}}*\dots*\vec{w}^1.
$
Denote by $T^{\vec{u}}$   the $d'_{\hat{L}}\times d'$ matrix   $(W_{k-j})_{k=1,\dots,d_{\hat{L}}',j=1,\dots,d'}$. Then, for $j=1,\dots,\tilde{d}$, the $jd'$-th row of $T^{\vec{u}}$ is exactly the $j$-th raw of $W$.   If we set $L=\lceil\frac{d'\tilde{d}}{s-1}\rceil$, then $\hat{L}\leq L$.  Taking $\vec{w}^{\hat{L}+1}=\dots=\vec{w}^{L}$ to be the delta sequence, we have $\vec{u}=\vec{w}^{L}*\dots*\vec{w}^1$. But Lemma  \ref{Lemma:convolution-to-matrix} implies 
$
    T^{\vec{u}}=T^{\vec{w}^{L}}\cdots T^{\vec{w}^1}.
$
Therefore the $d'j$th item of $T^{\vec{u}}x$ is the $j$th item of $W x$ for $j=1,\dots,\tilde{d}$.
Set
$  
    b^\ell=2^{\ell-1}B^\ell,\qquad \ell=1,2,\dots,L.
$  
Since  
\begin{equation}\label{convolution-and-Toep}
    \vec{w}^\ell*\vec{v}=T_{d_\ell,d_{\ell-1}}^{\vec{w}^\ell}\vec{v},
\end{equation}
it follows from Lemma \ref{Lemma:Induction} with $\ell=L-1$   that 
\begin{eqnarray*} 
 \sigma\circ \mathcal C^R_{L-1,\vec{w}^{L-1}, {b}^{L-1}} \circ \sigma  \circ \dots \circ \sigma\circ\mathcal C^R_{1,\vec{w}^1, {b}^1}(x)
 & = &
 \vec{w}^{L-1}* \cdots * \vec{w}^{1} *x +b^{L-1}{\bf 1}_{d_{L-1}}\nonumber \\
 &+&
 \sum_{k=1}^{L-2}T_{d_{L-1},d_{L-2}}^{\vec{w}^{L-1}}\cdots T_{d_{k+1},d_{k}}^{\vec{w}^{k+1}}
 b^k{\mathbf 1}_{d_k}.
\end{eqnarray*}
Define
\begin{eqnarray*}
   \vec{B}^{d_{L}}
    &:=& 
   \sum_{k=1}^{L}T_{d_{L},
   d_{L-1}}^{\vec{w}^{L}}\cdots T_{d_{k+1},d_{k}}^{\vec{w}^{k+1}}
   b^k{\mathbf 1}_{d_k}\\
   &=&
    w^{L}* 
   b^{L-1}{\bf 1}_{d_{L-1}}
  + 
  w^{L}*\sum_{k=1}^{L-2}T_{d_{L-1},d_{L-2}}^{\vec{w}^{L-1}}\cdots T_{d_{k+1},d_{k}}^{\vec{w}^{k+1}}
 b^k{\mathbf 1}_{d_k} 
\end{eqnarray*}
and $\vec{b}_{d_L}=(b_1,\dots,b_{d_L})^T$ as the vector  satisfying 
$$
  b_k:=\left\{\begin{array}{cc}
       \theta_k- (\vec{B}^{d_{L}})^{(k)}, & k=jd'   \\
       - (\vec{B}^{d_{L}})^{(k)}, &  \mbox{otherwise},
   \end{array}
   \right.
$$
where $(\vec{B}^{d_{L}})^{(k)}$ denotes the $k$-th compnent of the vector $(\vec{B}^{d_{L}})^{(k)}$.
Then, it follows from the definition of $\sigma$ and \eqref{def.down-sampling}   that
\begin{eqnarray}\label{shallow-represent}
 \mathcal S_{d_{L}',d'} 
  \sigma     ( \vec{w}^{L}* \cdots * \vec{w}^{1} *x
  +b_{d_{L}})
  =\sigma (Wx+\theta).
\end{eqnarray}
Let $d'=d$ and $\tilde{d}=2d+10$, we have from \eqref{shallow-represent} that there exist  $L_1=\left\lceil\frac{d(2d+10)}{s-1}\right\rceil$ sequences $\{\vec{w}^{\ell,1}\}_{\ell=1}^{L_1}$ and a vector $b_{d_{L_1},1}$
$$ 
 \mathcal S_{d+L_1s,d} 
  \sigma     ( \vec{w}^{L_1,1}* \cdots * \vec{w}^{1,1} *x
  +b_{d_{L_1},1})
  =\sigma (W_1x+\vec{\theta}_1).
$$
Similarly, for $j>1$, let $d'=\tilde{d}=2d+10$, there exist  $L_j=\left\lceil\frac{(2d+10)^2}{s-1}\right\rceil$ sequences $\{\vec{w}^{\ell,j}\}_{\ell=1}^{L_j}$ and   vectors $b_{d_{L_j},j}$
$$
   \mathcal S_{d+sL_j,2d+10} \sigma  ( \vec{w}^{L_j,j}* \cdots * \vec{w}^{1,j} *x
  +b_{d_{L_j},j})=\sigma(W_jx++\vec{\theta}_j).
$$
Therefore, for any $f\in L^\infty(\mathbb I^d)$, there holds
\begin{eqnarray}\label{upperbound-app}
    \mbox{dist}\left(Lip^{(r,c_0)},\Phi^{pool}_{L,s},L^\infty(\mathbb I^d)\right) 
 &\leq&
 \mbox{dist}\left(Lip^{(r,c_0)},
 \Phi_{L,2d+10,\dots,2d+10},L^\infty(\mathbb I^d)\right)\nonumber\\
 &\leq&
 \tilde{C}_1L^{-2r/d}\log^{2r/d} L.
\end{eqnarray}
This proves the upper bound of \eqref{almost-optimal-app} by noting the definition of the truncation operator $\pi_M$. We then turn to proving the lower bound. It follows from Lemma \ref{Lemma:covering-number} that 
  (\ref{covering condition}) in Lemma
\ref{Lemma:Relation c and l} is satisfied with $V=\pi_M\Phi_{L,s}^{pool}$, $\tilde{C}_1=1$,
$\beta=0$, $\tilde{C}_2=M$ and $n=c_1^*L^2\log L$. Hence, it follows from Lemma 
\ref{Lemma:Relation c and l} that
\begin{eqnarray*}
  \mbox{dist}(Lip^{(r,c_0)}_M,\pi_M\Phi_{L,s}^{pool},L_\infty(\mathbb I^d)) &\geq& 
    \mbox{dist}(Lip^{(r,c_0)}_M,\pi_M\Phi_{L,s}^{pool},L_1(\mathbb I^d))\\ 
      &\geq &
   C'
   \left[L^2\log L\log(L^2\log L)\right]^{-\frac{r}{d}} 
    \geq  
   C_1(L\log L)^{-\frac{2r}d}.
\end{eqnarray*}
This completes the proof of Theorem \ref{Theorem:optimal-app}. 
\end{proof}

To prove Theorem \ref{Theorem:optimal-learning}, we need the following concentration inequality  which can be found in
\cite[Theorem 11.4]{gyorfi2002distribution}.

\begin{lemma}\label{Lemma:CONCENTRATION INEQUALITY 1}
Assume $|y|\leq B$ and $B\geq 1$. Let $\mathcal F$ be a set of functions $f:\mathbb I^d\rightarrow\mathbb R$ satisfying $|f(x)|\leq B.$ Then for each $m\geq1$, with confidence at least
$$
      1-14\max_{x_1^m\in({\mathbb I^d})^m}\mathcal N_1\left(\frac{\beta\epsilon}{20B},\mathcal F,x_1^m\right)\exp\left(-\frac{\epsilon^2(1-\epsilon)\alpha m}{214(1+\epsilon)B^4}\right),
$$
there holds
 \begin{eqnarray*}
        \mathcal E(f)-\mathcal E(f_\rho)-( \mathcal E_D(f)-\mathcal E_D(f_\rho))  
      \leq 
     \epsilon(\alpha+\beta+ \mathcal E(f)-\mathcal E(f_\rho) ),\qquad \forall f\in\mathcal F,
\end{eqnarray*}
where $\alpha,\beta>0$, $0<\epsilon\leq1/2$ and $\mathcal E_D(f)=\frac1m\sum_{i=1}^m(f(x_i)-y_i)^2$.
\end{lemma}

We then prove Theorem \ref{Theorem:optimal-learning} as follows.

\begin{proof}
Due to \eqref{upperbound-app}, there exists a 
  $f^0_{L,s}\in\Phi_{L,s}^{pool}$ such that for any $f_\rho\in Lip^{(r,c_0)}$ there holds
\begin{eqnarray}\label{app-1.1}
  \mathcal E(f^{0}_{L,s})-\mathcal E(f_\rho)= \|f_\rho-f^{0}_{L,s}\|^2_{L^\infty(\mathbb I^d)}
  \leq 
 \tilde{C}_2L^{-4r/d}\log^{4r/d} L.
\end{eqnarray}
The classical error decomposition \cite{zhou2006approximation} then yields
\begin{eqnarray}\label{error-decomposition}
     \mathcal E(\pi_Mf_{D,L,s}^{pool})-\mathcal E(f_\rho)
   &\leq&
   \mathcal E(f^{0}_{L,s})-\mathcal E(f_\rho)+
   \overbrace{\mathcal E(\pi_Mf_{D,L,s}^{pool})-\mathcal E(f_\rho)-(\mathcal E_D(\pi_Mf_{D,L,s}^{pool})-\mathcal E_D(f_\rho))}^{\mathcal U_1} \nonumber\\
   &-&
    \overbrace{[\mathcal E(f^{0}_{L,s})-\mathcal E(f_\rho)-( \mathcal E_D(f^{0}_{L,s})-\mathcal E_D(f_\rho))]}^{\mathcal U_2} \nonumber\\
    &\leq&
    \tilde{C}_2L^{-4r/d}\log^{2r/d} L+\mathcal U_1+\mathcal U_2.
\end{eqnarray}
The bound of $\mathcal U_1$ is well known by noting that $\|f_{D,L,s}^{pool}\|_{L^\infty}\leq \|\|f_{D,L,s}^{pool}-f_\rho\|_{L^\infty}+M$. For example, it was derived in \cite[Lemma 7]{zhou2022learning} that for any $0<\delta<1$, with confidence $1-\delta$, there holds 
\begin{equation}\label{u1}
    \mathcal U_1\leq 4\log\frac2\delta \frac{(\|f_{D,L,s}^{pool}-f_\rho\|_{L^\infty}+M)\|f_{D,L,s}^{pool}-f_\rho\|_{L^\infty}}{\sqrt{m}}
    \leq
    c_1^* \frac{L^{-2r/d}\log^{2r/d} L}{\sqrt{m}}\log\frac2\delta,
\end{equation}
where $c_1^*$ is a constant depending only on $r,d,s,M$. We then turn to bounding $\mathcal U_2$ by using Lemma \ref{Lemma:CONCENTRATION INEQUALITY 1}. According to Lemma \ref{Lemma:covering-number}, we have for any $\beta>0$, 
$$
  \max_{x_1^m\in({\mathbb I^d})^m}\mathcal N_1\left(\frac{\beta\epsilon}{20B},\mathcal F,x_1^m\right)\leq \left(\frac{20M}{\beta\epsilon}\right)^{c^*_2  L^2\log L}
$$
for the constant $c_2^*$ depending only on $c^*$. Then Lemma \ref{Lemma:CONCENTRATION INEQUALITY 1} with $\mathcal F=\pi_M\Phi_{L,s}^{pool}$, $\varepsilon=1/2$ , $\beta=1/n$ yields that with confidence at least 
$$
    1-14\left({40Mm}\right)^{c^*_2  L^2\log L}\exp\{-c^*_3\alpha m\},
$$
there holds
$$
    \mathcal U_2\leq \frac12(\alpha+\frac1m+\mathcal E(\pi_Mf_{D,L,s}^{pool})-\mathcal E(f_\rho)),
$$
for some $c_3^*$ depending only on $M$.
Let 
$$
      14\left({40Mm}\right)^{c^*_2  L^2\log L}\exp\{-c^*_3\alpha m\}=
     \delta. 
$$
We have
$$
     \alpha=\frac{c_3^*}{m}\log\frac{14}\delta+\frac{c_3^*c_2^*L^2\log L\log (40Mm)}{m}.
$$
Therefore, with confidence $1-\delta$, there holds
\begin{equation}\label{u2bound}
 \mathcal U_2\leq   \frac{c_4^*L^2\log L\log  m}{m}\log\frac{2}\delta+\frac12(\mathcal E(\pi_Mf_{D,L,s}^{pool})-\mathcal E(f_\rho)),   
\end{equation}
where $c_4^*$ is a constant independent of $m$, $\delta$ or $L$. Inserting \eqref{u2bound} and  \eqref{u1} into \eqref{error-decomposition}, we get from $2ab\leq a^2+b^2$ for $a,b>0$ that 
\begin{eqnarray*} 
     \mathcal E(\pi_Mf_{D,L,s}^{pool})-\mathcal E(f_\rho)
    &\leq&
    2\tilde{C}_2L^{-\frac{4r}d}\log^{\frac{2r}d} L+ 2c_1^* \frac{L^{-\frac{2r}d}\log^{\frac{2r}d} L}{\sqrt{m}}\log\frac2\delta+\frac{2c_4^*L^2\log L\log  m}{m}\log\frac{2}\delta\\
    &\leq&
    c_5^*\left(L^{-\frac{4r}d}\log^{\frac{2r}d} L+\frac{L^2\log L\log  m}{m}\right)\log\frac{2}\delta
\end{eqnarray*}
for $c_5^*$ a constant independent of $m,L$  or $\delta$. Recalling $L\sim L^{\frac{d}{4r+2d}}$, we obtain that with confidence $1-\delta$,
$$
          \mathcal E(\pi_Mf_{D,L,s}^{pool})-\mathcal E(f_\rho)
          \leq c_6^* m^{-\frac{2r}{2r+d}}(\log m)^{\max\{2r/d,2\}}\log\frac2\delta,
$$
where $c_6^*$ is a constant independent of $m,L$  or $\delta$. This completes the proof of Theorem \ref{Theorem:optimal-learning}.
\end{proof}

To prove Corollary \ref{Corollary:optimal-learning}, we need the following well known probability to expectation formula. We present a simple proof for the sake of completeness.

\begin{lemma}\label{Lemma:prob-to-exp}
Let $0<\delta<1$, and  $\xi\in\mathbb R_+$ be a random variable. If $\xi\leq \mathcal A\log^b\frac{c}{\delta}$ holds with confidence $1-\delta$  for some $\mathcal A,b,c>0$, then
$$
      E[\xi]\leq c\Gamma(b+1) \mathcal A,
$$
where $\Gamma(\cdot)$ is the Gamma function.
\end{lemma}

\begin{proof}
Since $\xi\leq \mathcal A\log^b\frac{c}{\delta}$ holds with confidence $1-\delta$, we have
$$
    P[\xi>t]\leq c\exp\{\mathcal A^{-1/b}t^{1/b}\}.
$$
Using the probability to expectation formula
\begin{equation}\label{expectation formula}
               E[\xi] =\int_0^\infty P\left[\xi > t\right] d t
\end{equation}
  to the positive random variable $\xi$, we have
\begin{eqnarray*}
     E[\xi]\leq  c\int_{0}^\infty\exp\{\mathcal A^{-1/b}t^{1/b}\}
     \leq c\mathcal A\Gamma(b+1).
\end{eqnarray*}
This completes the proof of Lemma \ref{Lemma:prob-to-exp}. 
\end{proof}

We then prove Corollary \ref{Corollary:optimal-learning} as follows.

\begin{proof}[Proof of Corollary \ref{Corollary:optimal-learning}]
The lower bound of \eqref{almost-optimal-learning} is well known and we refer readers to \cite[Chap.3]{gyorfi2002distribution} for a detailed proof. The upper bound  of \eqref{almost-optimal-learning} follows from \eqref{learning-rate} and Lemma \ref{Lemma:prob-to-exp} with $\mathcal A=c_6^* m^{-\frac{2r}{2r+d}}(\log m)^{\max\{2r/d,2\}}$, $b=1$ and $c=2$ directly. This completes the proof of Corollary \ref{Corollary:optimal-learning}.
\end{proof}
\bibliographystyle{elsarticle-num}
\bibliography{optimal-dcnn}
\end{document}